\documentclass[12pt]{msml2020} % Include author names

\usepackage{xcolor,bm}
\newtheorem{assumption}{Assumption~}[section]

\newcommand{\bR}{\mathbb{R}}
\newcommand{\bS}{\mathbb{S}}
\newcommand{\bE}{\mathbb{E}}
\newcommand{\ba}{\bm{a}}
\newcommand{\bb}{\bm{b}}
\newcommand{\bx}{\bm{x}}
\newcommand{\by}{\mathbf{y}}
\newcommand{\bu}{\mathbf{u}}
\newcommand{\bv}{\mathbf{v}}

\newcommand{\cK}{\mathcal{K}}

\renewcommand{\SS}{\mathbb{S}}
\newcommand{\RR}{\mathbb{R}}

\title[Slow Deterioration]{The Slow Deterioration of the Generalization Error of the\\ Random Feature Model}
\usepackage{times}
\msmlauthor{%
 \Name{Chao Ma} \Email{chaom@princeton.edu}\\
 \addr PACM, Princeton University, Princeton, NJ 08544
 \AND
 \Name{Lei Wu} \Email{leiwu@princeton.edu}\\
 \addr PACM, Princeton University, Princeton, NJ 08544
 \AND 
 \Name{Weinan E} \Email{weinan@math.princeton.edu}\\
 \addr Department of Mathematics and PACM\\
 Princeton University, Princeton, NJ 08544
}

\begin{document}

\maketitle

\begin{abstract}
The random feature model exhibits a kind of resonance behavior
when the number of parameters is close to the training sample size. 
This behavior is characterized by the appearance of large generalization gap, and is 
due to the occurrence of very small eigenvalues for the associated Gram matrix.
In this paper, we examine the dynamic behavior of the gradient descent algorithm in this regime.
We show, both theoretically and experimentally, that there is a dynamic self-correction mechanism at work: 
The larger the eventual generalization gap, the slower it develops, 
{both because of the small eigenvalues.} 
This gives us ample time to stop the training process
and obtain solutions with good generalization property. 
\end{abstract}
\vspace*{2mm}

\section{Introduction}\label{sec:intro}

The properties of a machine learning model are largely controlled by its two most important parameters: the number of parameters $m$
and the training sample size $n$.  Of particular interest is when $m=n$,  the transition point between under- and over-parametrized regimes.
It has been reported that for some machine learning models, 
 the generalization error undergoes a ``double descent''  transition at this point:
 Roughly speaking, it increases as a function of $\gamma= m/n$ 
 before $\gamma=1$ and decreases  after $\gamma=1$ \citep{belkin2018reconciling}.  In particular, the generalization error is abnormally large around
 $\gamma =1$.

Figure \ref{fig: mnist-1} illustrates the result of fitting the  MNIST dataset using random feature models. Gradient descent (GD) initialized at $0$
 is used to train these models.
The blue dashed line is the test error for minimum-norm solutions, i.e. the GD solutions at $t=\infty$. 
As expected this curve shows a ``double descent'' behavior,  with a peak around $m=n$ \citep{belkin2018reconciling}.  
%Intuitively, the blow-up of the generalization error for minimum-norm solution is caused by the occurrence of small eigenvalues around $\gamma=1$. 
Shown in the red dashed line are the smallest eigenvalues of the associated Gram matrices as a function of  the number of features $m$. 
One can see that the  test error shows a  very strong negative correlation with the size of the smallest eigenvalues. 

\begin{figure}[!h]
\centering
\subfigure[]{%
    \label{fig: mnist-1}%
    \includegraphics[width=0.47\textwidth]{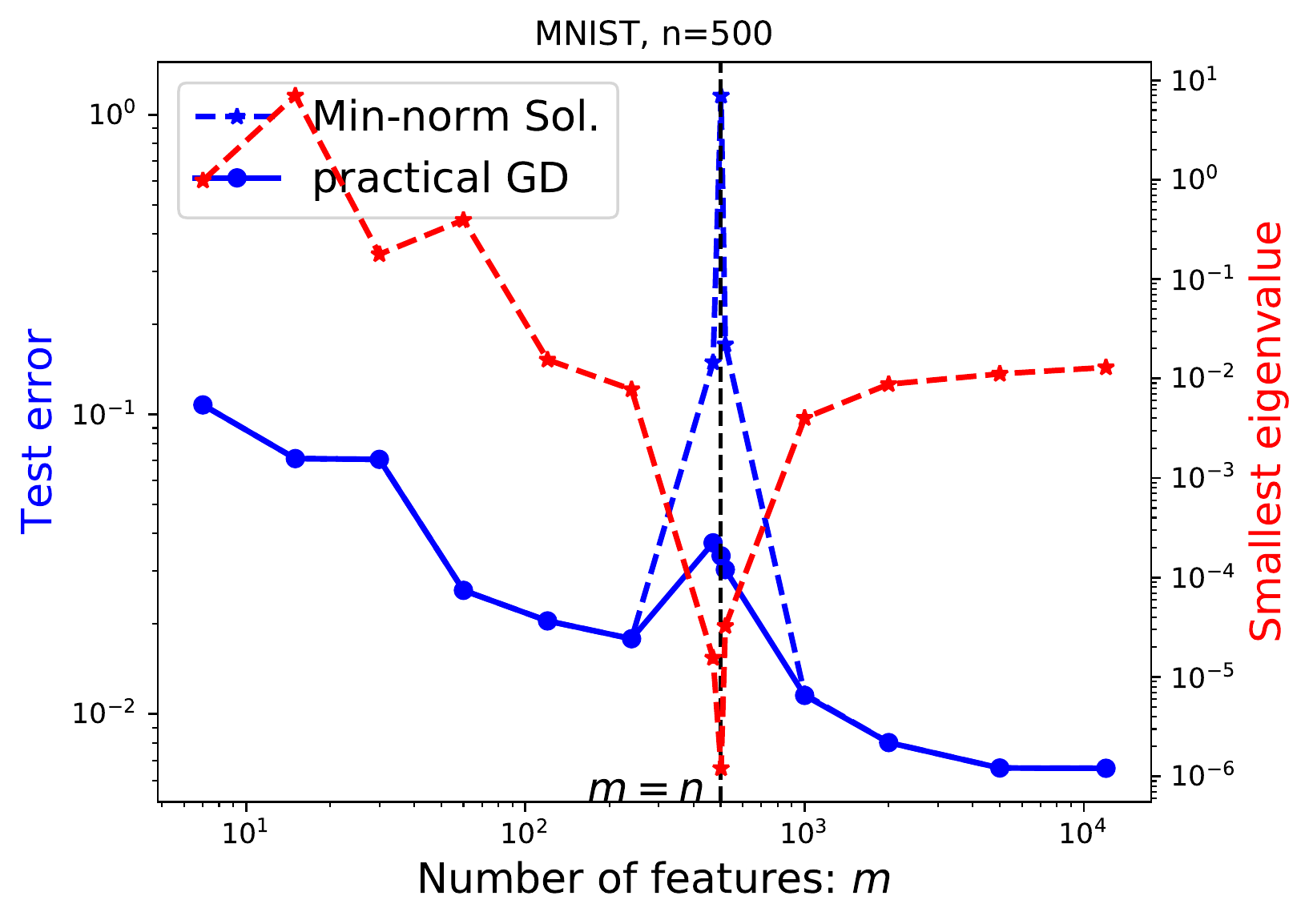}
}%
\subfigure[]{%
    \label{fig: mnist-2}%
    \includegraphics[width=0.42\textwidth]{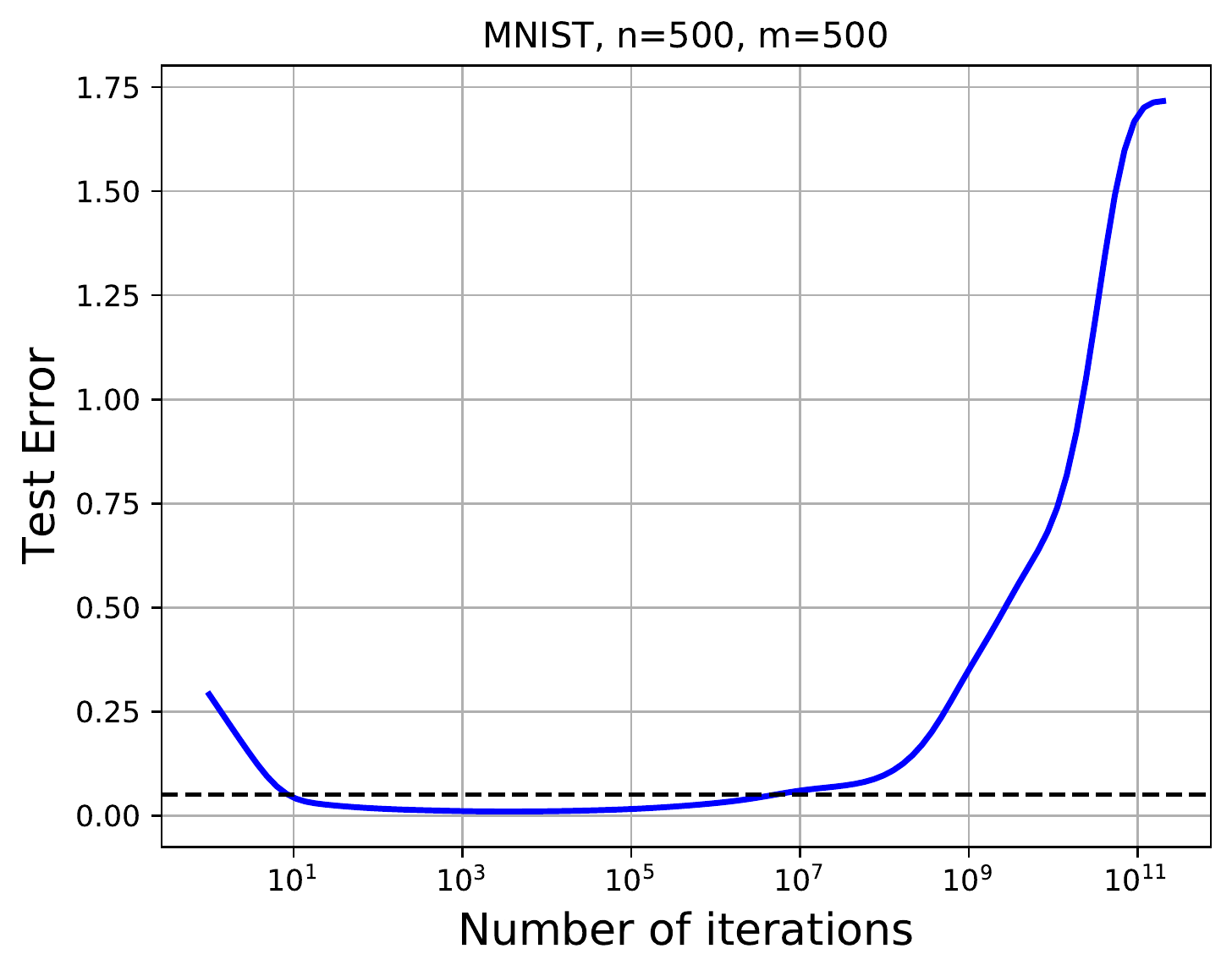}
}
\\
\subfigure[]{%
    \label{fig: mnist-3}
    \includegraphics[width=0.43\textwidth]{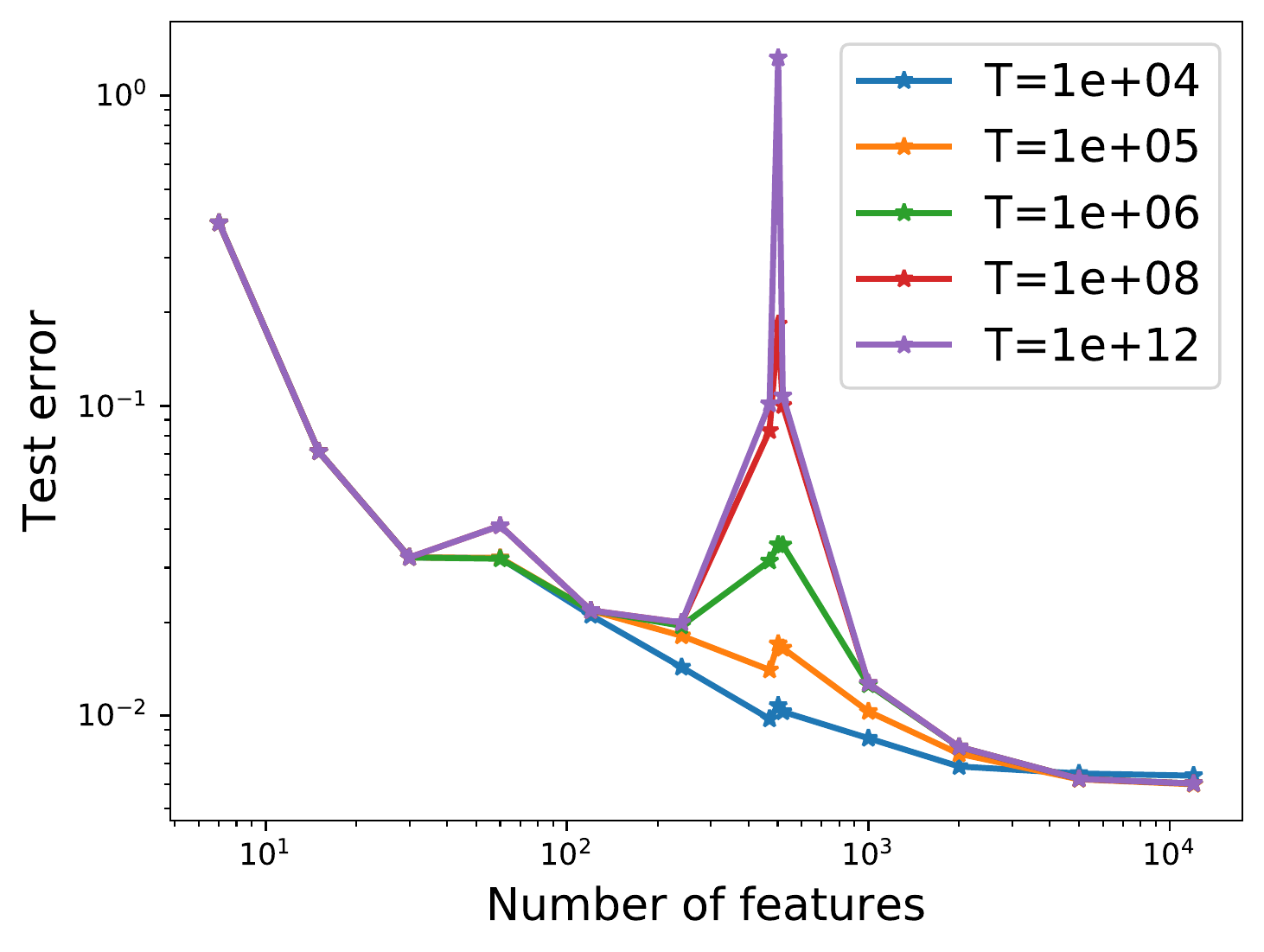}
}
\vspace*{-3mm}
\caption{\small Fitting the MNIST dataset using the random feature model $f(\bx;\bm{a})=\sum_{j=1}^m a_j \sigma(\bb_j^T \bx+c_j)$, with $(\bb_j,c_j)$ independently drawn from the uniform distribution over the $\ell_2$ unit sphere $\SS^d$, and $\sigma(t)=\max(0,t)$. Here  $d=784$, and we randomly chose $n=500$ samples from the class $0$ and $1$ to form the new training set. The learning rate of GD is $\eta=1/\lambda_{\max}$ , where $\lambda_{\max}$ is the largest eigenvalue of the associated Gram matrix.
(a) Test errors  and the smallest eigenvalues of the associated Gram matrices for different values of $m$. The blue dashed line shows test error of  the minimum-norm solutions. The blue solid line shows the result of the  GD solution found after $10^6$ iterations. (b)  Test error along the  GD path for $m=n=500$. The black dashed line corresponds to test error $0.05$. (c) The test error of GD solutions obtained by running different number of iterations.
}
\label{fig: mnist}
\end{figure}

Of more interest is the blue solid line
 in Figure \ref{fig: mnist-1}, which shows the test error of  ``practical'' GD solutions,  found by running GD for one million iterations. Roughly speaking, this curve shows a decreasing behavior as a function of $m$, with only a very small bump detected at $\gamma =1$.
   Note that no delicate early stopping strategy is used---we simply ran GD for one million iterations. 
   The reason behind this can be seen from  Figure \ref{fig: mnist-1} (b) which shows the test error along the GD dynamics at $\gamma=1$.
   One can see that over a very long period of time interval, the test error remains small before it finally increases to a very large value. 
   Another way to see this is shown in Figure \ref{fig: mnist-3}: the test error curves of GD solutions behave nicely for $T=10^4,10^5,10^6$. The ``singularity'' at $\gamma=1$ can only be seen for $T\geq 10^8$.

   The reason behind all these is the small eigenvalues of the Gram matrix. These small eigenvalues cause the test error of the minimum-norm
   solution to nearly blow up at $\gamma=1$. But at the same time, they also give rise to very slow dynamics for their effect to set in, i.e.
   it takes a long time for the large generalization errors  to develop dynamically.
A practitioner will likely terminate the iteration process before these large error starts to take an effect.

%This effect is most extreme at $\gamma=1$.  
%Away from thepoint, the smallest eigenvalues become larger.
%As a result, both effects, the large generalization error of the minimum-norm solution and the slow dynamics, are alleviated.
Another way to say this is that the effect of the ``double descent'' is washed out by the dynamics, as shown in the solid blue curve
in Figure \ref{fig: mnist-1} (a).

 In the following sections, we show that the {\bf slow deterioration} feature of the GD dynamics happens in more general cases (for general $\gamma$ and choice of kernels)
 and provide a theoretical analysis of this phenomenon. The analysis is based on the expansion of GD dynamics in the spectral space of the associated kernel operator. 
 The eventual poor generalization behavior is caused by the large norm of the solution, which in turn is caused by the GD dynamics
trying to finish the last mile for fitting the training data by resorting to the subspace of eigenvectors corresponding to the small
eigenvalues. This is indeed a form of over-fitting.  However, this over-fitting is not caused by over-parametrization, but rather 
by the GD dynamics trying to do its best to fit the training data.

\section{Related work}\label{sec:related}

The double descent phenomenon (also called jamming transition) was experimentally pointed out in~\citep{advani2017high,belkin2018reconciling,spigler2018jamming}. 
In~\citep{belkin2019two}, several simple models are analyzed to explain the double descent phenomenon. A more detailed analysis was carried out in \citep{mei2019generalization} with precise asymptotic results.
%computation of the generalization error using random matrix theory. The double descent phenomenon is also revealed by their results. 
Analysis of other models can be found  in~\citep{deng2019model,kini2020analytic}. However, these theoretical results do not consider the generalization error of solutions along the GD trajectory. 

As far as the GD dynamics is concerned, the existing literature \citep{yao2007early,suggala2018connecting,ma2019comparative,carratino2018learning} showed that there exist generalizable solutions along the GD path, which can be found by early stopping according to some specific rules.  As a comparison, our analysis of the slow deterioration phenomenon shows that there is a large interval in which the training process can be stoped without deteriorating the test accuracy.

%We emphasize that in the existing work on the generalization of GD solutions, usually either early stopping solution or the final solution is studied. Different from those existing work, our paper studies the generalization of the whole GD trajectory, and reveal the dynamical change of the generalization error along the trajectory. In this sense, 
Also relevant to this work is the ``frequency principle''  investigated in \citep{xu2018understanding,xu2019frequency,zhang2019explicitizing}.
It was observed that for random feature and neural network models, GD learns low-frequency components faster than high-frequency components. This can be viewed as the general mechanism behind the slow deterioration phenomenon.  

%Moreover, our work is  closely related to \cite{advani2017high}, which provided an analysis of the generalization error of GD solution for a linear regression model with the assumptions that both the  input date and noise are independently drawn from Gaussian distributions. 
The slow deterioration phenomenon was also observed in a simpler setting in  \citep{advani2017high}, but the analysis of \citep{advani2017high} still focused on optimal early stopping, for example how the optimal stopping time depends on the signal-to-noise ratio. Furthermore, their results cannot be directly applied in our setting, since the feature maps do not follow the Gaussian distribution assumption in that work. 
%In addition, we assume that the label $y$ is clean, and the essential noise in our setting comes from the approximation error, which is obviously not Gaussian and even correlated to the features.

\section{Experimental study}\label{sec:more_res}
\begin{figure}[!h]
\centering
\includegraphics[width=\textwidth]{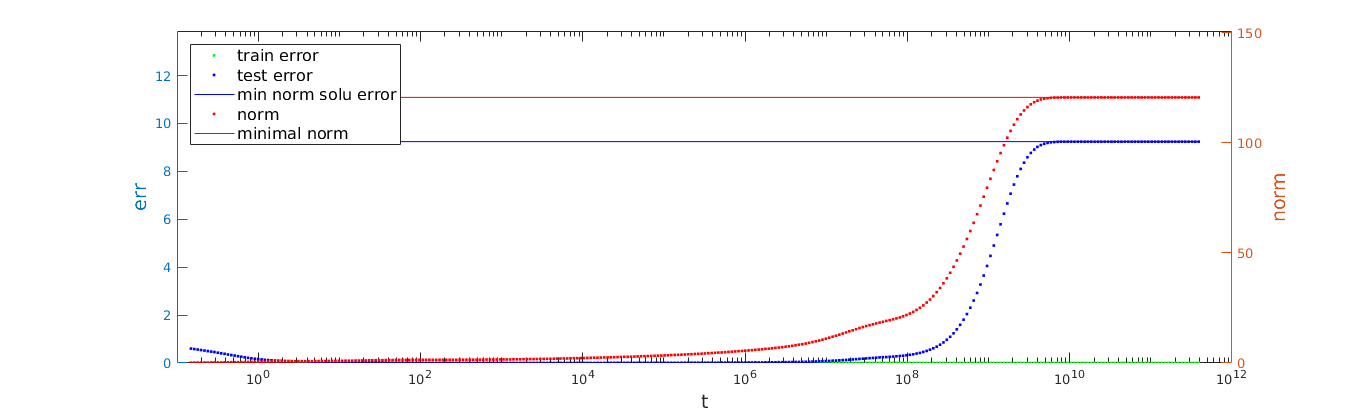}
\caption{\small The training error, testing error and the norm of the solution along the GD trajectory, initialized at zero. The random feature model
 is considered, with the feature being $\sigma(\bb^T\bx)$, where
 $\sigma$ is the ReLU function, $\bb$ and $\bx$ are both sampled uniformly from $\bS^{d-1}$. $d=10$, $m=n=500$. 
 The left axis shows the errors and the the axis on the right shows the norms.
 The results show that the testing error keeps being small from $t\approx 10^1$ to $t\approx10^6$, and approaches that of the minimal norm solution only after $t>10^9$.}\label{fig:rf_illus}
\end{figure}

%We can  compute the solution of the gradient flow at any given $t$. 
Consider a random feature model with features $\phi(\bx;\bb)$. Let $\{\bx_1,\bx_2,...,\bx_n\}$ be the training data sampled from distribution $\pi$, and $\{\bb_1,\bb_2,...,\bb_m\}$ be the random features sampled from $\mu$. Let $\Phi$ be an $n\times m$ matrix with $\Phi_{ij}=\phi(\bx_i;\bb_j)$. Then, the random feature model seeks to find a prediction function
\begin{equation}
\hat{f}(\bx) = \sum\limits_{k=1}^m a_k\phi(\bx;\bb_k),
\end{equation}
where $\ba=(a_1,b_2,...,a_m)^T$are the parameters. To find $\ba$,  GD is used to optimize the following least squares objective function,
\begin{equation}\label{eq:obj_fcn}
\min_{\ba\in\bR^m} \frac{1}{2n}\left\| \Phi\ba-\by \right\|^2,
\end{equation}
starting from the origin, where $\by$ is a vector containing the value of the target function at $\bx_i,\ i=1,2,...,n$. 
Instead of the discrete gradient descent algorithm, we will consider the gradient flow, which is the limit of  the gradient descent algorithm with 
learning rate tending to $0$. The dynamics of $\ba$ is then given by
\begin{equation}
    \frac{d}{dt}\ba(t) = -\frac{1}{m}\frac{\partial}{\partial\ba}\frac{1}{2n}\left\| \Phi\ba-\by \right\|^2 = -\frac{1}{mn}\Phi^T(\Phi\ba-\by).
\end{equation}

Let $\Phi=U\Sigma V^T$ be the singular value decomposition of $\Phi$, where $\Sigma=diag\{\lambda_1,\lambda_2,...,\lambda_{\min\{n,m\}}\}$ with $\lambda_i$ being the singular values of $\Phi$, in descending order. Then the GD solution of~\eqref{eq:obj_fcn} at time $t\geq0$ is
\begin{equation}\label{eq:solu}
\ba(t) = \sum\limits_{i:\lambda_i>0} \frac{1-e^{-\lambda_i^2t/(mn)}}{\lambda_i}(\bu_i^T\by)\bv_i.
\end{equation}
With this solution, we can conveniently compute the training and testing error at any time $t$, without doing GD. 
The equation above shows clearly the double-sided effect of small eigenvalues. On one hand, at $t = \infty$, they give an $O(1/\lambda)$ contribution
to the minimum-norm solution.
On the other hand,  for $t < O(mn/\lambda^2)$, their effect can be neglected.

We first consider the random feature model with the feature vector given by $\sigma(\bb^T\bx)$, where
 $\sigma$ is the ReLU function, $\bb$ and $\bx$ are both sampled uniformly from $\bS^{d-1}, d=10$.
% We will limit ourselves to the situation when the GD dynamics is initialized at $0$.  It is well-known
% that in this case, GD converges to the minimum-norm solution.
Figure~\ref{fig:rf_illus} shows the results for the case when $\gamma=1$.
Results for more general values of parameters are shown in Figure~\ref{fig:m_neq_n}. 
This figure suggests several things.  The first is that the level of overfitting 
for the minimum solution is drastically reduced away from the transition point. 
%We can see that when $m\neq n$, though the generalization error of the final minimal norm solution may still be larger than the best generalization performance along the trajectory, it is much smaller than that in the case of $m=n$. This alleviation of overfitting is very sensitive to the ratio of $m$ and $n$. Let $\gamma=m/n$, it turns out that 
In  fact the minimal norm solution performs much better even for $\gamma=0.9$ or $1.1$ and the degree of overfitting is much reduced.
In fact, when $\gamma=1/2$ or $2$, the  test error of the minimal norm solution is only a  few times larger than the best value
along the GD path.
If $m$ is chosen to be $O(\sqrt{n})$ or $O(n^2)$, no appreciable overfitting is observed:
% The minimal norm solution gives tbest generalization error along the GD trajectory. 
%In addition,  when overfitting is observed, the ascent of the testing error is much slower than the descent before that. 

\begin{figure}[!h]
\centering
\includegraphics[width=0.48\textwidth]{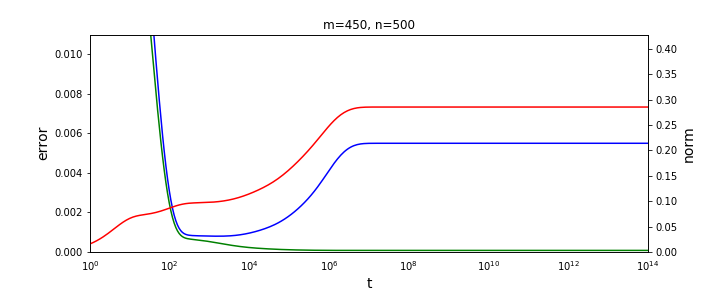}
\includegraphics[width=0.48\textwidth]{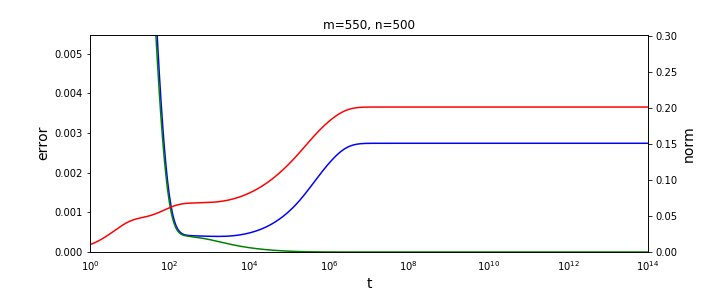}
\includegraphics[width=0.48\textwidth]{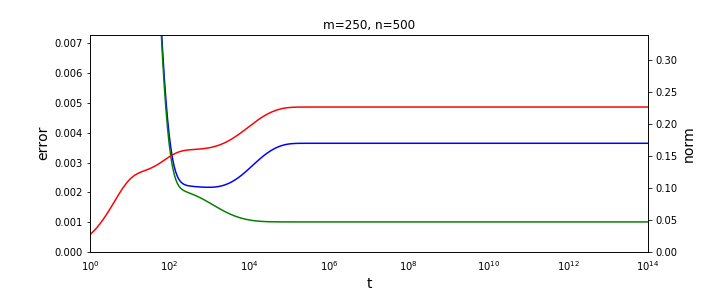}
\includegraphics[width=0.48\textwidth]{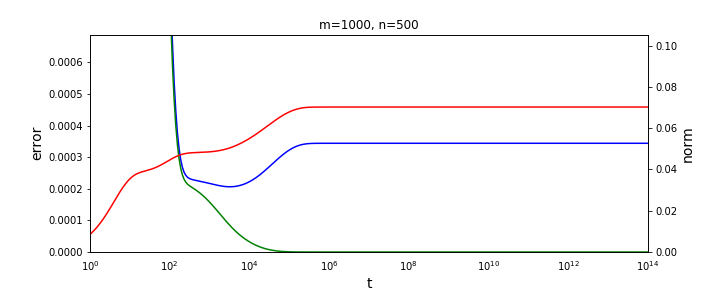}
\includegraphics[width=0.48\textwidth]{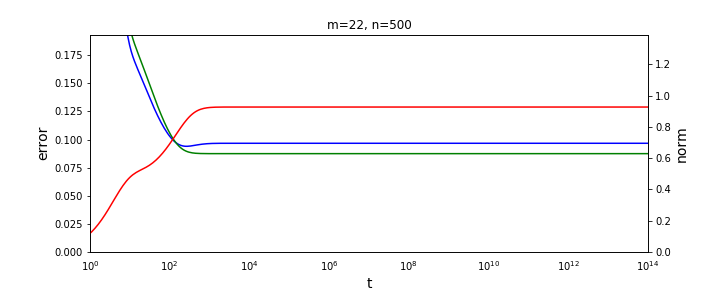}
\includegraphics[width=0.48\textwidth]{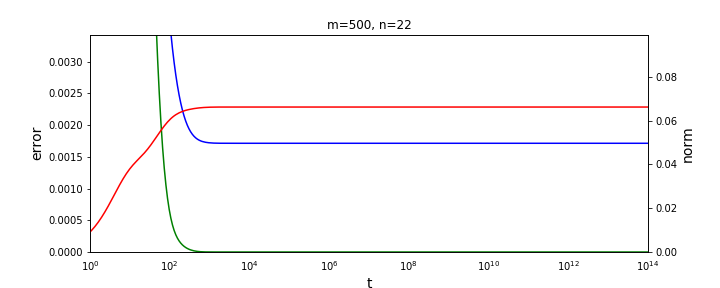}
\caption{\small Time history of the training error, testing error , and solution norm along the GD path for the general case. 
Green lines, blue lines and red lines denote  the training error, testing error, and solution norms, respectively. 
The minimal norm solutions perform better when $m$ is farther away from $n$. No appreciable overfitting is observed for $m=n^2$ or $\sqrt{n}$.}
\label{fig:m_neq_n}
\end{figure}

Another important observation is that the worse the overfitting, the slower it is reached. 
%This is seen from the very long ascending  process of the testing error after the minimal testing error is achieved. 
For example, in Figure~\ref{fig:rf_illus}, it takes nearly $10^{10}$ time units for GD to converge to the minimal norm solution. On the contrary, when $\gamma$ is far from $1$, the ascending process takes shorter time and the minimal norm solutions are reached earlier, as shown in Figure~\ref{fig:m_neq_n}. In Figure~\ref{fig:fix_time}, we put together testing error curves from many different experiments, with fixed $n$ but different $m$. The figure suggests that, though the testing errors of the final stable states are different, larger testing error always requires longer time to reach. Moreover, after the first period where the testing errors decrease, almost all the curves can be upper bounded by the function $y=c\sqrt{t}$,  shown in the figure by the dashed red line, where $c$ is a constant. This upper bound works uniformly for both under and over-parameterized models. It 
suggests that if we can tolerate a constant multiple of the minimal testing error, then we will have plenty of time to stop the GD algorithm
before getting a solution with bad testing error.
%Specifically, if it take time $t_0$ to get testing error $e$, then if we are satisfied with testing error $Ce$ for some constant $C$, then we will have a time interval with length at least $(C^2-1)t_0$ to do stopping. 

\begin{figure}[!h]
\centering
\includegraphics[width=0.8\textwidth]{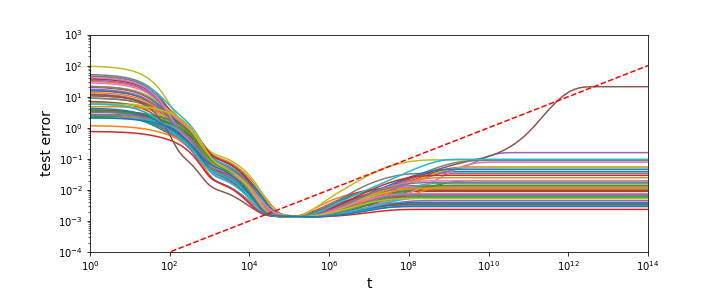}
\caption{\small Testing error curves for fixed $n$ ($n=500$) and different $m$'s. The red dashed line is the function $y=c\sqrt{t}$, where $c$ is a constant. %The dashed line can bound the testing error for nearly all the curves. 
%For those experiments that suffer to severe overfitting, the large test error happens lately.
}
\label{fig:fix_time}
\end{figure}

Figure~\ref{fig:ctrl_fcn2} shows the results for another random feature model, with feature $\mathbf{1}_{\{\bb^T\bx>0\}}$. The phenomenon described above also exists for this model, e.g. the larger the testing error, the longer it takes to reach.
In addition, the testing error curves can be uniformly bounded from above by a function proportional to $\sqrt{t}$. In this figure, we have translated the curves along the $y$-axis to eliminate the influence of the approximation error induced by different values of  $m$. 

\begin{figure}
\centering
\includegraphics[width=0.8\textwidth]{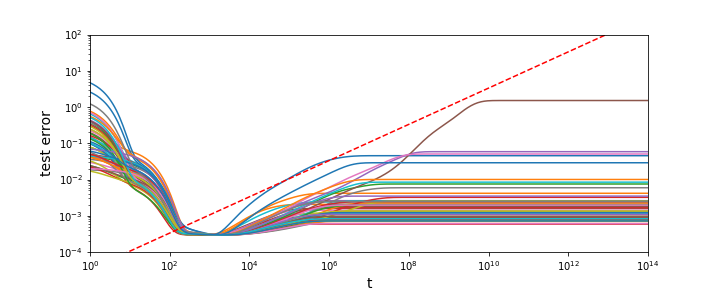}
\caption{\small Testing error curves for fixed $n$ ($n=500$) and different values of  $m$ for the model with features being given by $\mathbf{1}_{\{\bb^T\bx>0\}}$, in which $\bb$ and $\bx$ are still sampled from the uniform distribution on $\bS^{d-1}$. Note that the curves are translated along the $y$-axis so that all the curves have the same minimal value.}
\label{fig:ctrl_fcn2}
\end{figure}

\section{Theoretical analysis}\label{sec:theory}
In this section we explain the {slow deterioration} process of GD for the  random feature model from a theoretical perspective. 
%The intuitive cause of the slow overfitting of GD lies in the eigenvalues of the Gram matrix. In a word, 
Intuitively, since the parameters of the random feature model depend on the inverse of the Gram matrix, 
small eigenvalues of the Gram matrix give rises to minimum-norm solutions with large norms and consequently large testing error.
%caused by small eigenvalues of the Gram matrix. 
At the same time, the  GD dynamics proceeds slowly in the directions that correspond to these small eigenvalues. 
Hence, though the minimal norm solution may generalize poorly, it takes a long time for this behavior to set in along the GD path.

Let $B=[\bb_1,...,\bb_m]$, and with an abuse of notation let $\phi(\bx;B)=(\phi(\bx;\bb_1),\phi(\bx;\bb_2),...,\phi(\bx;\bb_m))^T$, recall the trajectory of the GD~\eqref{eq:solu}, the prediction function is given by
\begin{equation}\label{eq:pred_fcn}
\hat{f}_t(\bx) = \phi(\bx;B)^T\ba(t) = \sum\limits_{i:\lambda_i>0} \frac{1-e^{-\lambda_i^2t/mn}}{\lambda_i}(\bu_i^T\by)(\bv_i^T\phi(\bx;B)).
\end{equation}
Define a function $d(\lambda, t)$ for  $\lambda\geq0$ and $t\geq0$ by
\begin{equation}
d(\lambda, t)=\left\{ \begin{array}{lr}
\frac{1}{\lambda}\left(1-e^{\frac{\lambda^2t}{mn}}\right), & if\ \lambda>0, \\
0 & if\ \lambda=0,
\end{array}\right.
\end{equation}
and let $D(\Sigma, t)\in\bR^{n\times m}$ be a matrix obtained by applying $d(\cdot,t)$ elementwise to $\Sigma$, then~\eqref{eq:pred_fcn} can be written as 
\begin{equation}
\hat{f}_t(\bx) = \by^TUD(\Sigma,t)V^T\phi(\bx;B).
\end{equation}
Therefore, we have
\begin{align}
\| \hat{f}_t(\bx) \|_{l_2} & = \left\| \by^TUD(\Sigma,t)V^T\phi(\bx;B) \right\|_{l_2} \nonumber \\
  & \leq \left\| \by^TU \right\|_{2}\left\| D(\Sigma,t) \right\|_{2}\left\| V^T\phi(\bx;B) \right\|_{l_2} \nonumber \\
  & = \left\| \by \right\|_{2}\left\| D(\Sigma,t) \right\|_{2}\left\| \phi(\bx;B) \right\|_{l_2} \nonumber \\
  & \leq \left\| \by \right\|_{2}\left\| \phi(\bx;B) \right\|_{l_2} \max_{\lambda\geq0} \frac{1}{\lambda}\left(1-e^{\frac{\lambda^2t}{mn}}\right),\label{eq:ft}
\end{align}
where $\|\cdot\|_2$ denotes the $2$-norm of vectors for matrices, and $\|\cdot\|_{l_2}$ is the $l_2$ norm of functions based on the measure $\pi$. For the first part of the right hand side of~\eqref{eq:ft} we know it equals approximately to $\sqrt{n}\|f^*(\bx)\|_{l_2}$, where $f^*(\cdot)$ is the target function. For the second part, it equals approximately to $\sqrt{m\bE_{\bb\sim\mu}\|\phi(\bx;\bb)\|^2_{l_2}}$. For the third part, we have $(1-e^{-\lambda^2t/mn})/\lambda\leq1/\lambda$, and $(1-e^{-\lambda^2t/mn})/\lambda\leq\lambda t/mn$ because of $1-e^{-x}\leq x$. Hence, we have
\begin{equation}\label{eq:lamb_est}
\max_{\lambda\geq0} \frac{1}{\lambda}\left(1-e^{\frac{\lambda^2t}{mn}}\right) \leq \max_{\lambda\geq0} \min\left\{\frac{1}{\lambda}, \frac{\lambda t}{mn}\right\} \leq \sqrt{\frac{t}{mn}},
\end{equation}
and the equality is achieved at $\lambda=\sqrt{mn/t}$. Hence, approximately we have
\begin{equation}
\| \hat{f}_t(\bx) \|_{l_2}  \leq \|f^*(\bx)\|_{l_2}\sqrt{\bE_{\bb\sim\mu}\|\phi(\bx;\bb)\|^2_{l_2}}\sqrt{t}.
\end{equation}
More precisely, we have the following theorem.
\begin{theorem}\label{thm:rough_ft}
Assume there exists a constant $M$ such that $|f^*(\bx)|\leq M$ and $|\phi(\bx;\bb)|\leq M$ for any $\bb$ sampled from $\mu$. Then, for any $\delta>0$, with probability no less than $1-\delta$ over the choice of $\{\bx_i\}$ and $\{\bb_k\}$, we have
\begin{equation}\label{eq:rough_est}
\| \hat{f}_t(\bx) \|_{l_2}  \leq \left(\|f^*(\bx)\|^2_{l_2}+\sqrt{\frac{2M^2\log2/\delta}{n}}\right)^{1/2}\left(\bE_{\bb\sim\mu}\|\phi(\bx;\bb)\|^2_{l_2}+\sqrt{\frac{2M^2\log2/\delta}{m}}\right)^{1/2}\sqrt{t}.
\end{equation}
\end{theorem}

Theorem~\ref{thm:rough_ft} can be proved by simply applying the Hoeffding inequality to $\|\by\|_2^2$ and $\|\phi(\bx;B)\|_{l_2}^2$ and inserting the results into~\eqref{eq:ft}. From Theorem~\ref{thm:rough_ft}, we obtain the following corollary on the testing error along the GD trajectory. 

\begin{corollary}\label{cor:rough_est}
Consider a random feature model with feature $\phi(\bx;\bb)$ and target function $f^*(\bx)$. Assume that $\phi(\bx;\bb)$ and $f^*(\bx)$ satisfy the boundedness condition in Theorem~\ref{thm:rough_ft} with bound $M$. Then, for any $0\leq t<s$ and $\delta>0$, with probability no less than $1-\delta$ over the choice of the training data and random features, we have
{\small
\begin{align}
& \left\|\hat{f}_s(\bx)-f^*(\bx)\right\|_{l_2}  \leq \left\|\hat{f}_t(\bx)-f^*(\bx)\right\|_{l_2} \nonumber \\
  & + \left(\|f^*(\bx)\|^2_{l_2}+\sqrt{\frac{2M^2\log2/\delta}{n}}\right)^{1/2}\left(\bE_{\bb\sim\mu}\|\phi(\bx;\bb)\|^2_{l_2}+\sqrt{\frac{2M^2\log2/\delta}{m}}\right)^{1/2} \nonumber \\
  & \cdot\sqrt{s-t}.
\end{align}}
Hence, if there exists $t_0>0$ and $\epsilon>0$ such that $\left\|\hat{f}_{t_0}(\bx)-f^*(\bx)\right\|_{l_2}\leq \epsilon$, then for any $t>t_0$ we have
{\small
\begin{align}
\left\|\hat{f}_s(\bx)-f^*(\bx)\right\|_{l_2} &\leq \left(\|f^*(\bx)\|^2_{l_2}+\sqrt{\frac{2M^2\log2/\delta}{n}}\right)^{1/2}\left(\bE_{\bb\sim\mu}\|\phi(\bx;\bb)\|^2_{l_2}+\sqrt{\frac{2M^2\log2/\delta}{m}}\right)^{1/2}\sqrt{t-t_0} \nonumber \\
 & +\epsilon.
\end{align}}
\end{corollary}

Corollary~\ref{cor:rough_est} can be proved by using
\begin{equation}
\left\|\hat{f}_s(\bx)-f^*(\bx)\right\|_{l_2} \leq \left\|\hat{f}_t(\bx)-f^*(\bx)\right\|_{l_2} + \left\|\hat{f}_s(\bx)-\hat{f}_t(\bx)\right\|_{l_2},
\end{equation}
and then bounding the second term on the right hand side using similar techniques as  for Theorem~\ref{thm:rough_ft}. 

These results show that the growth of the testing error in the overfitting regime is controlled by a square root function. 
%Hence, the speed of overfitting cannot be very fast.  
%For the cases where the  minimal norm solution generalizes poorly, the overfitting always happens at late times. 
The square root function is the result of adding up  the exponential contributions from all the small eigenvalues.
However, these results are not optimal, in the sense that it treats all the eigenvalues and eigenvectors equally. In reality the leading eigenvalues 
closely approximate the corresponding eigenvalues of the kernel operator, and the associated eigenvectors are also close to that of the
kernel operator  evaluated at the training dataset.
If we further assume that the target function concentrates on the low-frequency modes, i.e. it mainly lies in the subspace spanned by the leading eigenfunctions of the kernel space, then we will have $\sum_{i=1}^p (\bu_i^T\by)^2\approx \|\by\|^2$ for a small $p$, and $\sum_{i=p+1}^n (\bu_i^T\by)^2$ is close to $0$. 
Since the leading eigenvalues are large and  do not contribute to  overfitting, when estimating the overfitting effect we may limit ourselves to the eigenvalues $\lambda_i$ with $i\geq p+1$. This improves the bound for the testing error. 

Specifically, let $K(\cdot,\cdot)$ be the kernel induced by the random features $\phi(\bx;\bb)$ defined by
\begin{equation}\label{eq:def_kernel}
K(\bx,\bx') = \bE_{\bb\sim\mu} \phi(\bx;\bb)\phi(\bx';\bb),
\end{equation}
and let $\cK$ be the corresponding kernel operator:
\begin{equation}
\cK f(\bx) = \int_\Omega K(\bx,\bx')f(\bx')\pi(d\bx').
\end{equation}
Then for bounded $\phi$, $\cK$ is a trace class operator and has eigenvalues $\{\mu_i\}_{i=1}^\infty$ and eigenfunctions $\{\psi_i(\cdot)\}_{i=1}^\infty$. We assume that the eigenvalues are in descending order. {For convenience, we assume $m=n$, and} let $G=\frac{1}{n^2}\Phi\Phi^T$.  By the SVD decomposition of $\Phi$ we have $G=\frac{1}{n^2}U\Sigma\Sigma^TU^T$. Hence, the $\bu_i$'s are the eigenvectors of $G$ and  the $\lambda_i^2/n^2$'s are the eigenvalues. \{For the target function, we assume $f^*$ is the first eigenfunction of $\cK$.
\begin{assumption}\label{assump:fstar}
Assume 
\begin{equation}
f^*(\bx) = \psi_1(\bx),
\end{equation}
% and $r(\bx)$ satisfies $\langle r, \psi_k \rangle_{l_2(\pi)}=0$ for $k=0,1,...,q$ and $\|r(\bx)\|_{l_2}\leq\epsilon$. 
% Furthermore, we assume $\left|\|\by\|/\sqrt{n}-\|f^*\|\right|<\delta$, which is a high-probability event noting that $\|\by\|^2/n$ is a Monte-Carlo approximation of $\|f^*\|^2$.
\end{assumption}
In the following assumption, we assume that the top eigen-pairs of $G$ are close to that of $\cK$ (the eigenvectors of $G$ are close to the eigenfunctions of $\cK$ evaluated at the training data). This is numerically verified in the Appendix.
\begin{assumption}\label{assump:spec}
Assume there exists a constant $C$, such that
\begin{itemize}
{
\item[1] $\left|\frac{\|\by\|^2}{n}-1\right|\leq \frac{C}{\sqrt{n}}$
\item[2] $\left|\frac{\bu_1^T\by}{\sqrt{n}}-1\right|<\frac{C}{\sqrt{n}}$ and $\left\| \frac{\sqrt{n}}{\lambda_1} \bv_1^T\phi(\bx;B) - \psi_1(\bx) \right\|_{l_2}<\frac{C}{\sqrt{n}}$
\item[3] $\left|\langle \frac{\sqrt{n}}{\lambda_i}\bv_i^T\phi(\bx;B), \frac{\sqrt{n}}{\lambda_j}\bv_j^T\phi(\bx;B) \rangle - \delta_{i,j}\right|<\frac{C}{\sqrt{n}}$, for any $2\leq i,j \leq \lfloor \sqrt{n} \rfloor$,
where $\langle\cdot,\cdot\rangle$ is the $l_2$ inner product of two functions. 
}
% \item[1] $\left|\frac{1}{mn}\lambda_k^2-\mu_k\right|<\delta$
% \item[2] $\left\|\bu_k-\frac{\psi_k(X)}{\sqrt{n}}\right\|_2<\delta$, where $\psi_k(X)=(\psi_k(\bx_1),...,\psi_k(\bx_n))^T$, and $\left|\frac{\bu_k^T\by}{\sqrt{n}}-c_k\right|<\delta$
% \item[3] $\left\| \frac{\sqrt{n}}{\lambda_i} \bv_k^T\phi(\bx;B) - \psi_k(\bx) \right\|_{l_2}<\delta$
\end{itemize}
\end{assumption}
The first assumption in~\ref{assump:spec} is just a Monte-Carlo approximation of $\|f^*\|_{l_2}^2$ and is easy to verify. 
The second assumption in~\ref{assump:spec} characterizes the approximation of the largest eigenvalue and the corresponding eigenvector of $G$ to that of $\cK$. 
The third assumes that functions $\frac{\sqrt{n}}{\lambda_i}\bv_i^T\phi(\bx;B)$ are nearly orthogonal for $2\leq i\leq \lfloor \sqrt{n} \rfloor$.
If we multiply $\bu_k^T$ on both side of the equation $\Phi\Phi^T=U\Sigma V^T\Phi$, we get
\begin{equation}
\lambda_k^2\bu_k^T = \bu_k^T\Phi\Phi^T = \bu_k^TU\Sigma V^T\Phi^T=\lambda_k\bv_k^T\Phi.
\end{equation}
{Hence, for $i=1,2,...,n$, we have $\bv_k\phi(\bx_i;B)/\lambda_k=\bu_{k,i}$, which gives 
\begin{equation}
    \frac{1}{n}\sum\limits_{k=1}^n \frac{\sqrt{n}}{\lambda_i}\bv_i^T\phi(\bx_k;B)\cdot\frac{\sqrt{n}}{\lambda_j}\bv_j^T\phi(\bx_k;B)=\delta_{i,j},
\end{equation}
and this makes the second assumption of~\ref{assump:spec} quite reasonable.
}
% In Assumption~\ref{assump:fstar}, $\langle \cdot,\cdot \rangle_{l_2(\pi)}$ denotes the inner product of functions in $l_2(\pi)$. In the following we omit the subscript $l_2(\pi)$. 
%By the assumptions on the residual $r(\bx)$, we know $c_k=\langle f^*, \psi_k \rangle$. Based on Assumptions~\ref{assump:spec} and~\ref{assump:fstar}, we have the following theorem.

{
Then, we have the following results on the error between $\hat{f}_t$ and $f^*$:
\begin{theorem}\label{thm:finer_est}
Let $\hat{\lambda}_i = \frac{\lambda_i}{n}$. If Assumption~\ref{assump:spec} holds, then for sufficiently large $n$ such that $\frac{C}{\sqrt{n}}<1$, we have
\begin{equation}\label{eq:thm}
\|\hat{f}_t-f^*\|\leq 3e^{-2\hat{\lambda}_1^2t} + (5C+1+2\sqrt{C}Md(t))^2n^{-\frac{1}{2}},
\end{equation}
where $M^2=\frac{1}{n}\int \|\sigma(B^T\bx)\|^2\pi(d\bx)$, and
\begin{equation*}
d(t) = \min\left\{ \sqrt{t}, \hat{\lambda}_{\lfloor\sqrt{n}\rfloor+1}t, \hat{\lambda}_n^{-1} \right\}. 
\end{equation*}
\end{theorem}

\begin{proof}
First, by~\eqref{eq:pred_fcn} we have 
\begin{align}
\|\hat{f}_t-f^*\| &= \left\| \sum\limits_{i=1}^n \frac{1-e^{-\hat{\lambda}_i^2t}}{\lambda_i} (\bu_i^T\by)(\bv_i^T\phi(\bx;B)) - \psi_1(\bx)) \right\| \nonumber \\
  &\leq \left\| \frac{1-e^{-\hat{\lambda}_1^2t}}{\lambda_1} (\bu_1^T\by)(\bv_1^T\phi(\bx;B)) - \psi_1(\bx) \right\| + \left\| \sum\limits_{i=2}^{\lfloor\sqrt{n}\rfloor} \frac{1-e^{-\hat\lambda_i^2t}}{\lambda_i} (\bu_i^T\by)(\bv_i^T\phi(\bx;B)) \right\| \nonumber \\
  & \quad + \left\| \sum\limits_{i=\lfloor\sqrt{n}\rfloor+1}^n \frac{1-e^{-\hat\lambda_i^2t}}{\lambda_i} (\bu_i^T\by)(\bv_i^T\phi(\bx;B)) \right\| \nonumber \\
  &=: I+J+K. \label{eq:err_decomp}
\end{align}
Next, we estimate $I$, $J$ and $K$ separately.

For $I$, by Assumption~\ref{assump:spec}, we get
\begin{align}
I & \leq \left\| e^{-\hat\lambda_1^2t}\psi_1(\bx) \right\| + \left\| (1-e^{-\hat\lambda_1^2t})\left(\frac{1}{\lambda_1}(\bu_1^T\by)(\bv_1^T\phi(\bx;B))-\psi_1(\bx)\right) \right\| \nonumber\\
  & \leq e^{-\hat\lambda_1^2t} + \left\| \left(\frac{1}{\lambda_1}(\bu_1^T\by)(\bv_1^T\phi(\bx;B))-\psi_1(\bx)\right) \right\| \nonumber \\
  & \leq e^{-\hat\lambda_1^2t} + \left\|\left(\frac{\bu_1^T\by}{\sqrt{n}}-1\right)\frac{\sqrt{n}\bv_1^T\phi(\bx;B)}{\lambda_1}\right\| + \left\|\frac{\sqrt{n}\bv_1^T\phi(\bx;B)}{\lambda_1}-\psi_1(\bx)\right\| \nonumber \\
  & \leq e^{-\hat\lambda_1^2t} + \frac{C}{\sqrt{n}}(1+\frac{C}{\sqrt{n}}) + \frac{C}{\sqrt{n}} \nonumber \\
  & \leq e^{-\hat\lambda_1^2t} + 3\frac{C}{\sqrt{n}}.
\end{align}

For $J$, we consider $J^2$ and have
\begin{align}
J^2 & = \sum\limits_{i,j=2}^{\lfloor\sqrt{n}\rfloor} (1-e^{-\hat\lambda_i^2t}) (1-e^{-\hat\lambda_j^2t}) \frac{\bu_i^T\by}{\sqrt{n}}\frac{\bu_j^T\by}{\sqrt{n}} \left\langle \frac{\sqrt{n}}{\lambda_i}\bv_i^T\phi(\bx;B), \frac{\sqrt{n}}{\lambda_j}\bv_j^T\phi(\bx;B)  \right\rangle \nonumber \\
  & \leq \sum\limits_{i,j=2}^{\lfloor\sqrt{n}\rfloor}\left|\frac{\bu_i^T\by}{\sqrt{n}}\frac{\bu_j^T\by}{\sqrt{n}}\right|\left|\left\langle \frac{\sqrt{n}}{\lambda_i}\bv_i^T\phi(\bx;B), \frac{\sqrt{n}}{\lambda_j}\bv_j^T\phi(\bx;B) \right\rangle\right|.
\end{align}
By Assumption~\ref{assump:spec} we have
\begin{align}
J^2 &\leq \frac{C}{\sqrt{n}}\sum\limits_{i,j=2}^{\lfloor\sqrt{n}\rfloor}\left|\frac{\bu_i^T\by}{\sqrt{n}}\frac{\bu_j^T\by}{\sqrt{n}}\right| + \sum\limits_{i=2}^{\lfloor\sqrt{n}\rfloor} \frac{(\bu_i^T\by)^2}{n} \nonumber\\
  & \leq \frac{C+1}{n} \sum\limits_{i=2}^{\lfloor\sqrt{n}\rfloor} (\bu_i^T\by)^2 \nonumber\\
  & \leq (C+1) \left( \frac{\|\by\|^2}{n}-\frac{(\bu_1^T\by)^2}{n} \right) \nonumber \\
  & \leq (C+1) \left(\left|\frac{\|\by\|^2}{n}-1\right|+\left|\frac{(\bu_1^T\by)^2}{n}-1\right|\right) \nonumber \\
  & \leq (C+1)\frac{4C}{\sqrt{n}}. 
\end{align}
Hence, for $J$ we obtain $J\leq (2C+1)n^{-1/4}$.

Finally, we estimate $K$. Let $\title{U}=[\bu_{\lfloor\sqrt{n}\rfloor+1}, \cdots, \bu_n]$ and $\title{V}=[\bv_{\lfloor\sqrt{n}\rfloor+1}, \cdots, \bv_n]$. Then, we can bound $K$ as 
\begin{equation}
    K \leq \max_{\lfloor\sqrt{n}\rfloor+1\leq i\leq n}\frac{1-e^{-\hat\lambda_i^2t}}{\lambda_i} \left\|\frac{\tilde{U}^T\by}{\sqrt{n}}\right\|\cdot\left\|\sqrt{n}\tilde{V}^T\phi(\bx;B)\right\|.
\end{equation}
For $\left\|\frac{\tilde{U}^T\by}{\sqrt{n}}\right\|$, similar to the estimate of $J$ we have
\begin{equation}
    \left\|\frac{\tilde{U}^T\by}{\sqrt{n}}\right\|\leq \sqrt{\frac{\|\by\|^2}{n}-\frac{(\bu_1^T\by)}{n}} \leq 2\sqrt{C}n^{-1/4}.
\end{equation}
For $\left\|\sqrt{n}\tilde{V}^T\phi(\bx;B)\right\|$,  we have
\begin{equation}
\left\|\sqrt{n}\tilde{V}^T\phi(\bx;B)\right\|\leq \sqrt{n}\left\|V^T\phi(\bx;B)\right\|=\sqrt{n}\|\phi(\bx;B)\|\leq nM.
\end{equation}
For $\max_{\lfloor\sqrt{n}\rfloor+1\leq i\leq n}$, by~\eqref{eq:lamb_est} we have
\begin{equation}\label{eq: K1}
    \max_{\lfloor\sqrt{n}\rfloor+1\leq i\leq n}\frac{1-e^{-\hat\lambda_i^2t}}{\lambda_i} \leq \frac{\sqrt{t}}{n}.
\end{equation}
On the other hand, it is easy to see 
\begin{equation}\label{eq: K2}
    \max_{\lfloor\sqrt{n}\rfloor+1\leq i\leq n}\frac{1-e^{-\hat\lambda_i^2t}}{\lambda_i} \leq \max_{\lfloor\sqrt{n}\rfloor+1\leq i\leq n}\frac{1}{\lambda_i}\leq \frac{\hat\lambda_n^{-1}}{n},
\end{equation}
and 
\begin{equation}\label{eq: K3}
    \max_{\lfloor\sqrt{n}\rfloor+1\leq i\leq n}\frac{1-e^{-\hat\lambda_i^2t}}{\lambda_i} \leq \max_{\lfloor\sqrt{n}\rfloor+1\leq i\leq n}\frac{\hat\lambda_i t}{n}.
\end{equation}
Combining~\eqref{eq: K1}-\eqref{eq: K3}, we have
\begin{equation}
    \max_{\lfloor\sqrt{n}\rfloor+1\leq i\leq n}\frac{1-e^{-\hat\lambda_i^2t}}{\lambda_i} \leq \frac{1}{n}d(t),
\end{equation}
and for $K$ we have estimate
\begin{equation}
    K\leq 2\sqrt{C}n^{-1/4}Md(t).
\end{equation}

Combining the estimates for $I$, $J$ and $K$ completes the proof. 
\end{proof}

From the theorem above, we see that the testing error in time can be divided into three regimes. The first regime is an exponential decay regime
 governed by the first term on the right hand side of~\eqref{eq:thm}. The second regime is a period in which the testing error keeps  being small, after the decaying term has been sufficiently reduced and before the effect represented by the term with $d(t)$ really shows up. In the third regime, the last term on the right hand side of~\eqref{eq:thm} begins to manifest  and finally becomes very large (because $1/\lambda_n$ is very small). The existence of the second regime is caused by the gap between the leading eigenvalue $\lambda_1$ and the eigenvalues appearing in the third term, and this period is long when this gap is large, which occurs when $n$ is large. The following corollary roughly characterizes the length of this period under a specific assumption 
 about the decay rate of $\hat{\lambda}_t$.
\begin{corollary}
%If $\hat{\lambda}_k$ decays in rate no less than $1/\sqrt{k}$, i.e. 
If there exist constant $C'$ such that $|\hat{\lambda}_k|\leq\frac{C'}{\sqrt{k}}$ holds for any $1\leq k\leq\lfloor\sqrt{n}\rfloor+1$. Then, there exist constants $c_1$ and $c_2$ (which may depend on $\hat{\lambda}_1$) such that  
\begin{equation*}
    \|\hat{f}_t-f^*\|^2 \leq \frac{c_1}{\sqrt{n}},
\end{equation*}
when $c_2\log n\leq t\leq c_2n^{\frac{1}{4}}$. Hence, the length of the second regime described above is at least in the order of $n^{1/4}$.
\end{corollary}

\begin{proof}
Let $c_2 = \frac{1}{4\lambda_1^2}$, then, for $t$ that satisfies $c_2\log n\leq t\leq c_2n^{\frac{1}{4}}$, we have
\begin{align}
\|\hat{f}_t-f^*\| &\leq n^{-\frac{1}{4}} + (5C+1+2\sqrt{C}M\hat\lambda_{\lfloor\sqrt{n}\rfloor+1}c_2n^{\frac{1}{4}})n^{-\frac{1}{4}}.
\end{align}
By the condition that $|\hat\lambda_k|\leq\frac{C'}{\sqrt{k}}$ we have
\begin{equation}
\|\hat{f}_t-f^*\| \leq n^{-\frac{1}{4}}(1+5C+1+2\sqrt{C}MC'c_2).
\end{equation}
The proof is completed by taking $c_1=2+5C+2\sqrt{C}C'c_2M$.
\end{proof}
}

\begin{remark}
Theorem~\ref{thm:finer_est} shows that the estimates of the testing error can be refined if the target function is band-limited, and the leading 
eigen-pairs of the Gram matrix align well with that of the kernel operator. In the appendix we numerically demonstrate that the assumptions for this theorem are likely to hold in practice. 
\end{remark}
%\begin{remark}
%The estimates in Theorem~\ref{thm:finer_est} consist of three parts. The first part is large at the beginning but converges to $0$ exponentially fast.
%The second part characterizes the overfitting, and increases as $\sqrt{t}$. The third term is a small error term  caused mainly by the error in the approximation of the spectrum of the Gram matrix to that of the kernel operator.  
%\end{remark}
\begin{remark}
%The upper bound of the testing error, which increases as $\sqrt{t}$, is uniform for all the models that satisfy the assumptions. For specific cases, the testing error will not increase forever. 
The testing error of the final solution (the minimal norm solution) is largely controlled by the smallest eigenvalue of the Gram matrix. In the appendix we demonstrate by numerical experiments that the minimal eigenvalue of the Gram matrix is very small only when $m\approx n$.
\end{remark}

\section{Discussions}\label{sec:discuss}

In this paper, we studied the double-sided effect of the small eigenvalues of the Gram matrix in a random feature model.  
An obvious question is how small
 the smallest eigenvalues are,  In the appendix, we provide numerical evidence that the smallest eigenvalue approximately obeys the predictions
 of the Marchenko-Pastur distribution  in the random matrix theory when $\gamma$ is close to 1.
 
The second obvious question is why should the smallest eigenvalues be so small when $\gamma=1$, i.e. why should the
Gram matrix be almost singular when $\gamma=1$.  While we do not have a simple answer to this question, we can address the related question:  Why should the Gram matrix be non-singular away from $\gamma=1$? 
The reason is that for large values of $\gamma$, the spectrum of the Gram matrix  converges to the top spectrum of the corresponding kernel operator, which is obviously non-singular. This can be seen from  the last figure in Figure \ref{fig: spectrum}.   The same is true for very small values of $\gamma$ if we consider the ReLU feature $\sigma(\bb^T\bx)$. In this case, the non-zero eigenvalues of $G=\Phi\Phi^T/mn$ equal to the eigenvalues of $\Phi^T\Phi/mn$, and the latter can be viewed as a Gram matrix with $\bx$ being the features and $\bb$ being the data, and $\gamma$ much larger than $1$, due to the symmetry of $\bb$ and $\bx$.

% {\bf should we put in more details?}

The third question is how relevant this phenomenon is for other models, such as the two-layer neural network model? We guess similar ``resonance'' phenomenon also exists, but extensive experimental and theoretical study is left for future work.

% {\bf here we look at the spectrum of the linearized operator for the two-layer neural network model}

% % Acknowledgments---Will not appear in anonymized version
% \acks{We thank a bunch of people.}
% \bibliographystyle{plain}
\bibliography{slow_overf}

\appendix 

\section{Spectrum of the Gram matrix}
Let $\pi$ be the uniform distribution on the unit sphere $\bS^{d-1}$, where $d\geq3$ is the dimension. The random feature model we consider has features $\phi(\bx;\bb)=\sigma(\bb^T\bx)$, where $\bx\sim\pi$, $\bb\sim\pi$, and $\sigma(\cdot)$ is the ReLU function. Then, the kernel defined by~\eqref{eq:def_kernel} have a closed form
\begin{equation}\label{eq:kernel_form}
k(\bx,\bx') = \sqrt{1-(\bx^T\bx')^2}+\bx^T\bx'(\pi-\arccos(\bx^T\bx')).
\end{equation}
We study the following three quantities. 
\begin{itemize}
    \item Gram matrix $G=(G_{i,j})$ with $G_{i,j}=\frac{1}{nm}\sum_{s=0}^m \sigma(\bb_s^T\bx_i)\sigma(\bb_s^T\bx_j)$.
    \item Kernel matrix $K=(K_{i,j})$ with $K_{i,j}=\frac{1}{n}k(\bx_i,\bx_j)$. 
    \item The kernel operator $\mathcal{K}$ defined by 
    \[
        \mathcal{K} f(\bx) = \int_{\SS^{d-1}} k(\bx,\bx') f(\bx') d\pi(\bx').
    \]
\end{itemize}

In Figure \ref{fig: spectrum}, the spectra of Gram matrices for various value of $m$ are displayed. As a comparison, the spectra of the corresponding kernel operator and kernel matrix are also plotted. Let $\gamma=m/n$.  We have the following observations.
\begin{figure}[!h]
\centering
\subfigure[$\gamma=0.5$]{
    \includegraphics[width=0.32\textwidth]{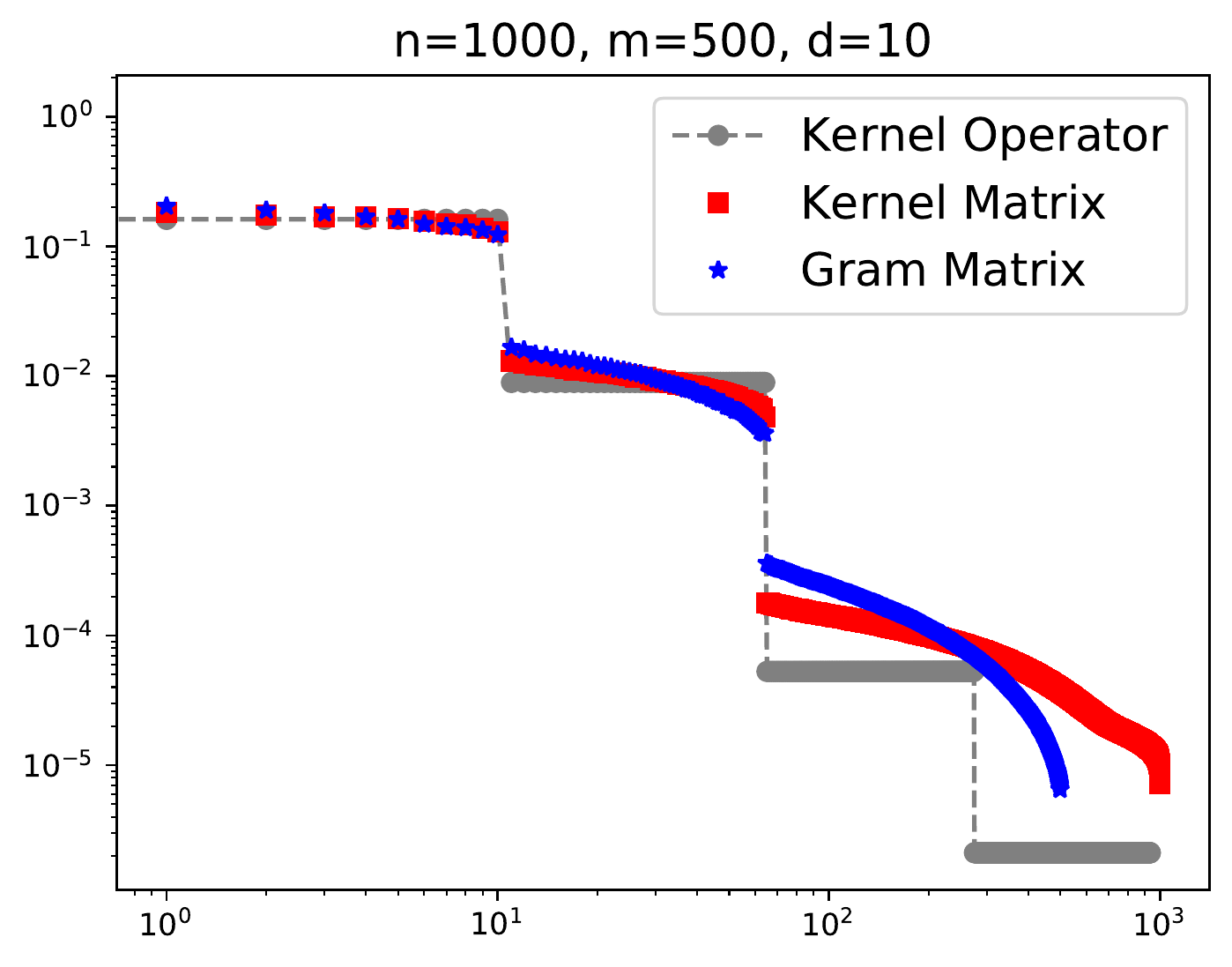}
}\hspace*{-2mm}
\subfigure[$\gamma=0.8$]{
\includegraphics[width=0.32\textwidth]{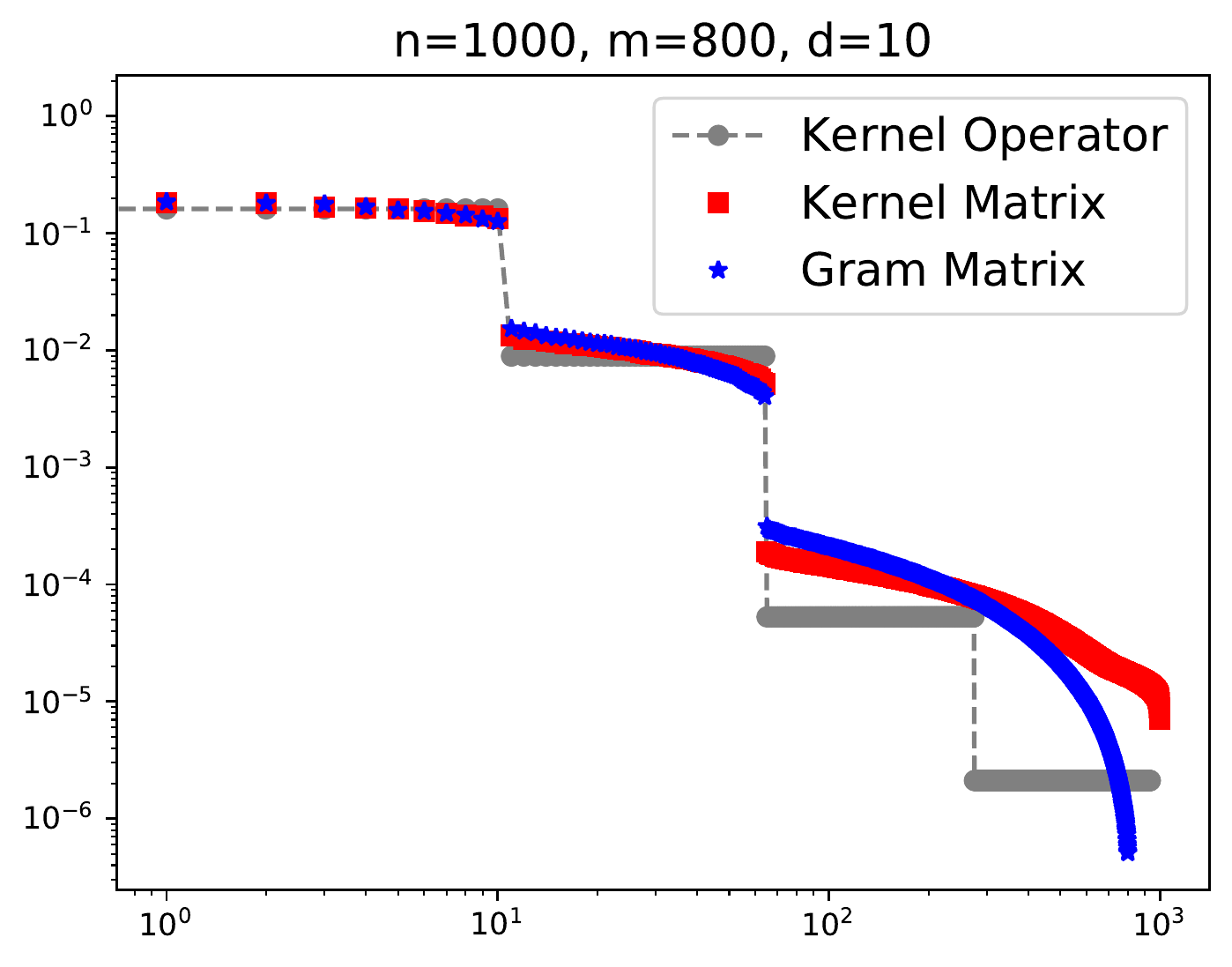}
}\hspace*{-2mm}
\subfigure[$\gamma=1.0$]{
\label{fig: spectrum-1}
\includegraphics[width=0.32\textwidth]{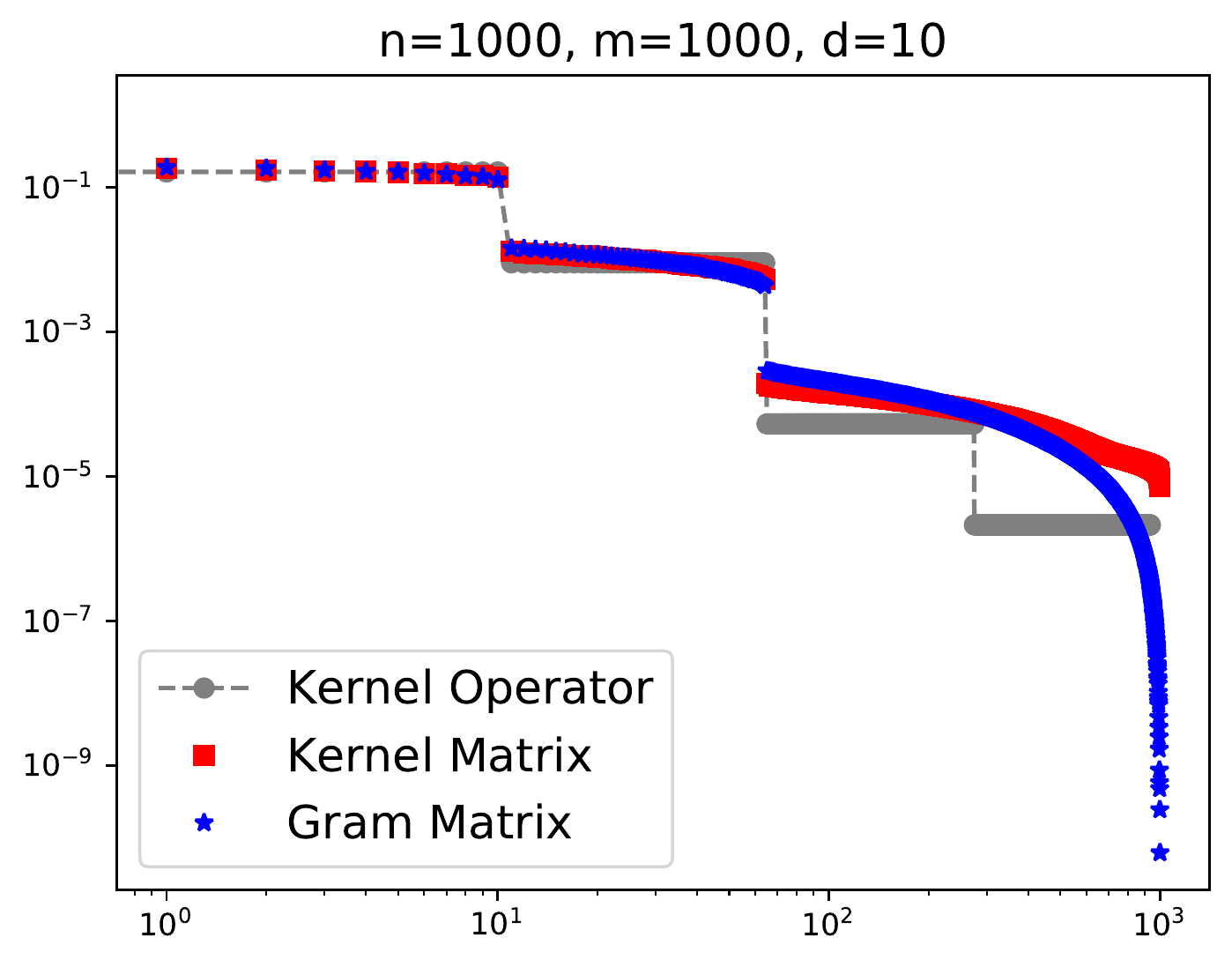}
}
\\
\subfigure[$\gamma=1.5$]{
    \includegraphics[width=0.32\textwidth]{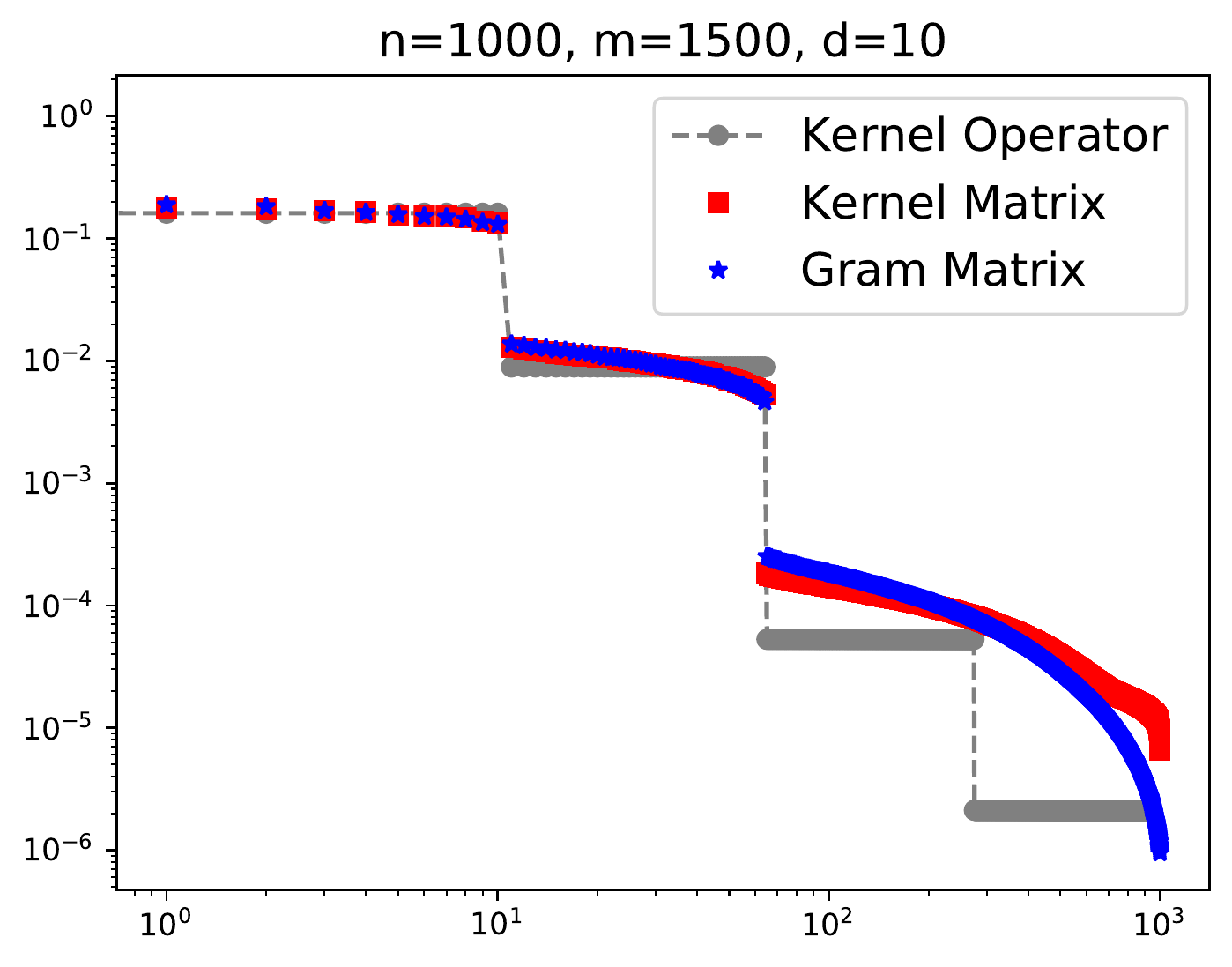}
}\hspace*{-2mm}
\subfigure[$\gamma=2.0$]{
\includegraphics[width=0.32\textwidth]{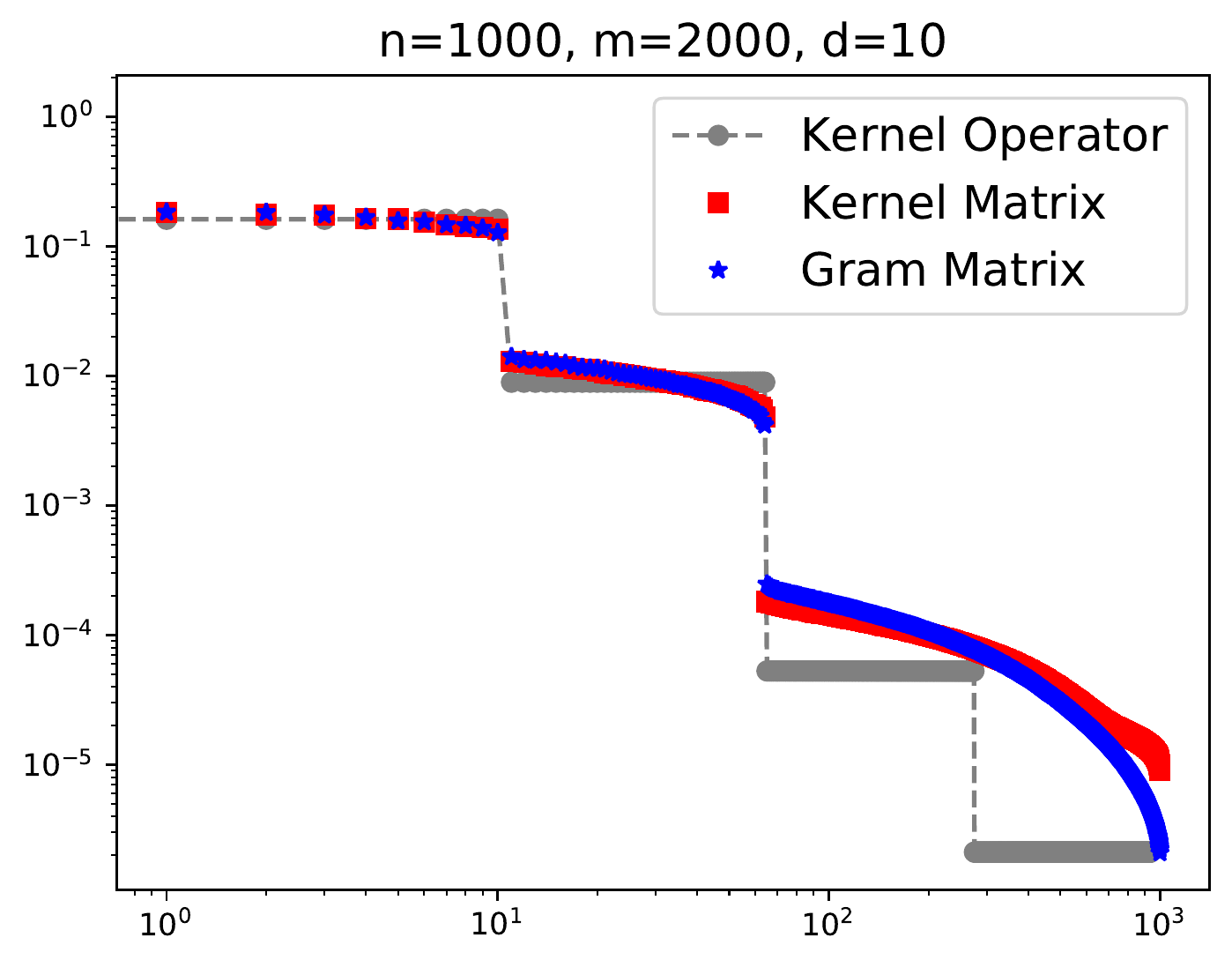}
}\hspace*{-2mm}
\subfigure[$\gamma=8.0$]{
\includegraphics[width=0.32\textwidth]{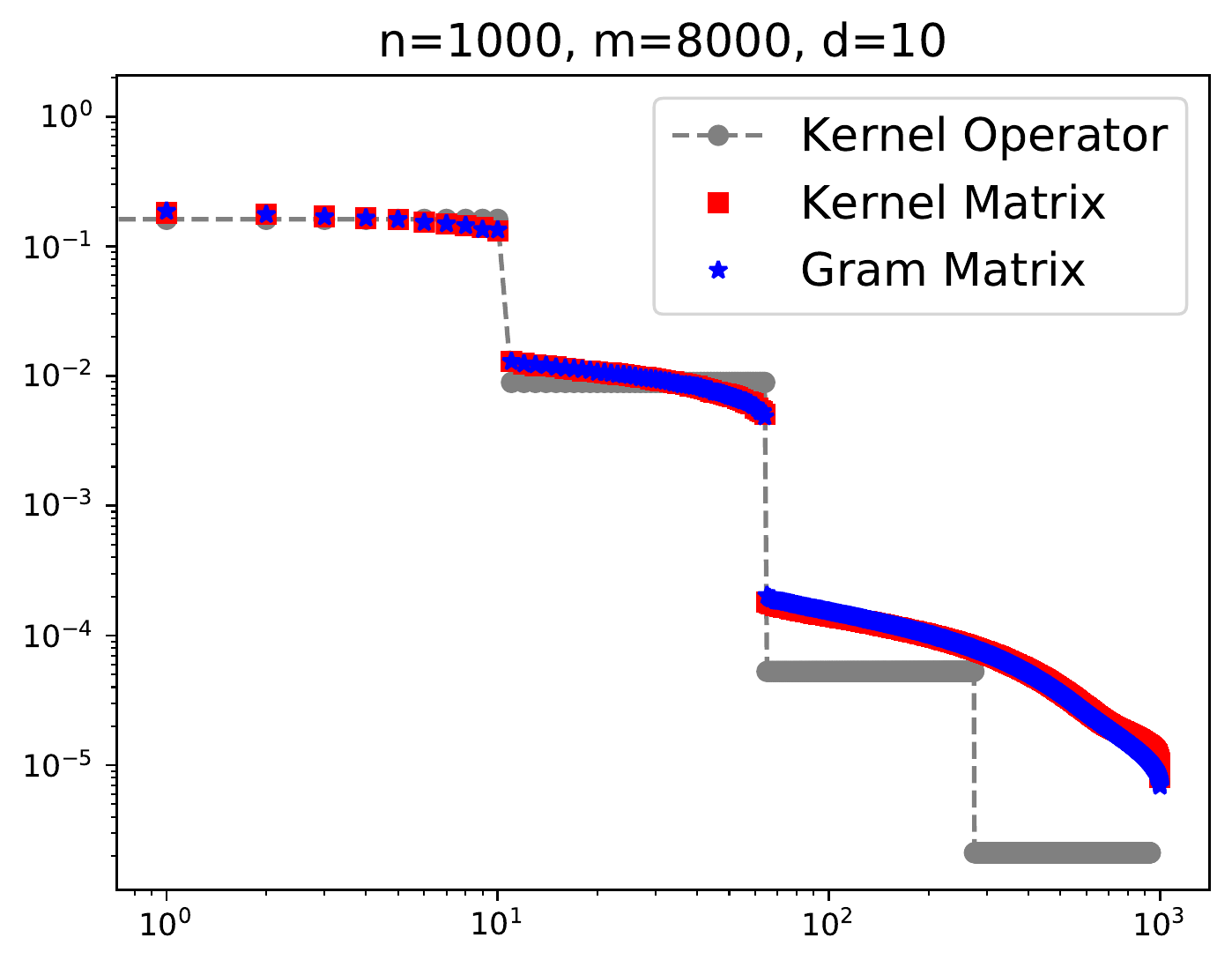}
}
\vspace*{-2mm}
\caption{\small Comparison of the spectra of the kernel operator, kernel matrix and the corresponding Gram matrix. We  see that  the spectrum of Gram matrix consist of large discrete spectrum and small ``continuous'' spectrum. The former are very close to the top eigenvalues  of the kernel operator, while the latter is rather random especially when $\gamma\approx 1$. Moreover, as $\gamma \to \infty$, the spectrum of the Gram matrix converges to the spectrum of the kernel matrix.}
\label{fig: spectrum}
\end{figure}

\begin{itemize}
\item The spectrum of the kernel operator is stage-like.

\item The large eigenvalues of kernel matrix approximate the corresponding eigenvalues of kernel operator very well. The small eigenvalues are away from the smallest eigenvalues of kernel operator, but still relatively large.

\item The spectrum of the Gram matrix is close to the spectrum of the kernel matrix when $\gamma$ is large. For example, the two spectra are almost indistinguishable at $\gamma=8$.

\item The spectrum of the Gram matrix at $\gamma = 1$ has  a very long tail at $0$, i.e. the tail is extremely small and have outliers close to zero. However, the tail of the spectrum for $\gamma\neq 1$ does not have outliers near  zero.
\end{itemize}
In the following, we provide some theoretical evidence why the spectrum of behave like this.

\subsection{The stage-like spectrum of the kernel operator}
For this special case, in the following we show that the spectrum of kernel operator $\cK$ has explicit formula.
Since $k(\bx,\bx')$ only depends on the inner product of its inputs, and $\bx$ are uniformly distributed on the sphere, we know that the eigenfunctions of $K$ are spherical harmonics, and the eigenvalues for spherical harmonics with the same frequency are the same. Let $\{\psi_j(\bx)\}_{j=0}^\infty$ be the sequence of spherical harmonics on $\bS^{d-1}$, with non-decreasing frequencies, and $\mu_j$ be the corresponding eigenvalues, then $k(\bx,\bx')$ can be represented as 
\begin{equation}
k(\bx,\bx') = \sum\limits_{j=0}^\infty \mu_j \psi_j(\bx)\psi_j(\bx').
\end{equation}
% Next, we assume the target function $f^*$ has the following decomposition,
% \begin{equation}
% f^*(\bx) = \sum\limits_{j=0}^\infty c_j\psi_j(\bx). 
% \end{equation}
% Then, we easily have $\|f^*\|_2 = \sum_j c_j^2$ and $\|f^*\|_{\cH_k}=\sum_j c_j^2/\mu_j$, where $\|\cdot\|_2$ is the $l_2$ norm and $\|\cdot\|_{\cH_k}$ means the RKHS norm. 

For $\cK$, we have the following theorem about its eigenvalues and eigenfunctions.
\begin{theorem}\label{thm:specK}
Let $\{Y_{n,i}: n\geq0,\ 1\leq i\leq N(d,n)\}$ be the spherical harmonics on $\bS^{d-1}$, with $N(d,0)=1$ and 
\begin{equation}
N(d,n) = \frac{2n+d-2}{n}\left(\begin{array}{c}
n+d-3 \\ n-1
\end{array}\right),
\end{equation}
for $n\geq1$. Then, $Y_{n,i}(\cdot)$ are eigenfunctions of $\cK$ for any $(n,i)$, and the corresponding eigenvalues only depend on $n$. Let the eigenvalue associated with $Y_{n,i}$ be $\lambda_n$, then we have 
\begin{equation}\label{eq:lam0}
\lambda_0 = \frac{2\sqrt{\pi}d\Gamma(\frac{d}{2})}{\Gamma(d)\Gamma(\frac{d-1}{2})},
\end{equation}
and $\lambda_n=C(d)\Lambda(d,n)$ for $n\geq1$, where
\begin{align}
C(d) &= 2^{\frac{d-5}{2}}\pi^{\frac{2d+3}{4}}d(d-2)\left[\frac{\Gamma(\frac{d-2}{2})\Gamma(d-1)}{\Gamma(\frac{d-1}{2})\Gamma(\frac{d}{2})}\right]^{\frac{1}{2}},\\
\Lambda(d,n) &= \frac{2^{n-\frac{1}{2}}\Gamma(\frac{n+d-2}{2})}{(n+d-3)!(n+d-1)!\Gamma(\frac{3-n}{2})^2\Gamma(\frac{n+d-1}{2})}. \label{eq:Lambda}
\end{align}
\end{theorem}

\begin{proof}
By the close form of the kernel~\eqref{eq:kernel_form}, the kernel $k(\bx,\bx')$ only depends on the inner product of $\bx$ and $\bx'$. Hence, by~\citep{xie2016diverse}, we know that spherical harmonics are eigenfunctions of $\cK$, and the eigenvalues are the same for spherical harmonics with the same order $n$. Hence, $K(\bx,\bx')$ has the following decomposition,
\begin{equation}\label{eq:expressionK}
K(\bx,\bx') = \sum\limits_{n=0}^\infty \lambda_n\sum\limits_{i=1}^{N(d,n)} Y_{n,i}(\bx)Y_{n,i}(\bx').
\end{equation}
A direct integral gives~\eqref{eq:lam0}. For $n\geq1$, let $P_n(t)$ be the Legendre Polynomials defined in~\citep{frye2012spherical}, then by Theorem 4.11 thereof, 
\begin{equation}\label{eq:expressionP}
P_n(\bx^T\bx') = \frac{\Omega_{d-1}}{N(d,n)}\sum\limits_{i=1}^{N(d,n)} Y_{n,i}(\bx)Y_{n,i}(\bx'),
\end{equation}
where $\Omega_{d-1}$ is the surface area of $\bS^{d-1}$. Hence, combining~\eqref{eq:expressionK} and~\eqref{eq:expressionP}, we have
\begin{align}
\lambda_n &= \frac{1}{\Omega_{d-1}}\int\int_{\bx,\bx'}K(\bx,\bx')P_n(\bx^T\bx')d\Omega_{d-1}d\Omega_{d-1} \\
  &= \frac{1}{\Omega_{d-1}}\int_{\bx}K(\bx,\bx')P_n(\bx^T\bx')d\Omega_{d-1} \\
  &= \frac{1}{\Omega_{d-1}}\int_{-1}^1 k(t)P_n(t)(1-t^2)^{\frac{d-3}{2}}dt, \label{eq:lamn}
\end{align} 
where we let $k(t)=\sqrt{1-t^2}+t(\pi-\arccos(t))$. By~\citep{frye2012spherical}, $P_n$ and $P_m$ are orthogonal when $m\neq n$, and 
\begin{equation}
\int_{-1}^1 (1-t^2)^{\frac{d-3}{2}}P_n(t)^2dt = \frac{\Omega_{d-1}}{\Omega_{d-2}N(d,n)}.
\end{equation}
Hence, if we let
\begin{equation}\label{eq:P_and_C}
C_n(t)=\left[ \frac{\Omega_{d-2}N(d,n)2^{3-d}\pi\Gamma(n+d-2)}{\Omega_{d-1}(n+\frac{d-2}{2})\Gamma(n+1)\Gamma(\frac{d-2}{2})^2} \right]^{\frac{1}{2}}P_n(t),
\end{equation}
then $C_n(t)$ are Gegenbauer polynomials defined through the generating function
\begin{equation}\label{eq:generating}
\sum\limits_{n=0}^\infty C_n(t)s^n = \frac{1}{(1-2st+s^2)^{\frac{d-2}{2}}}.
\end{equation}
According to~\eqref{eq:generating} and~\citep{erdlyi1954tables}, we can compute the following integrals,
\begin{equation}\label{eq:int_Gegen1}
\int_{-1}^1(1-t^2)^{\frac{d-2}{2}}C_n(t)dt = \frac{\pi^{3/2}2^{n-2}(d-2)\Gamma(\frac{n+d-2}{2})}{n!\Gamma(\frac{1-n}{2})\Gamma(\frac{3-n}{2})\Gamma(\frac{n+d+1}{2})},
\end{equation}
and 
\begin{equation}\label{eq:int_Gegen2}
\int_{-1}^1(1-t^2)^{\frac{d-3}{2}}t(\pi-\arccos t)C_n(t)dt=\frac{\pi^{3/2}2^{n-3}(d-2)(n^2+(d-2)n+1)\Gamma(\frac{n+d-2}{2})}{(n+d-1)n!\Gamma(\frac{3-n}{2})^2\Gamma(\frac{n+d+1}{2})}.
\end{equation}
Combining~\eqref{eq:int_Gegen1} and~\eqref{eq:int_Gegen2}, we have 
\begin{equation}\label{eq:int_Gegen3}
\int_{-1}^1(1-t^2)^{\frac{d-3}{2}}k(t)C_n(t)dt = \frac{\pi^{3/2}d(d-2)2^{n-2}\Gamma(\frac{n+d-2}{2})}{(n+d-1)^2n!\Gamma(\frac{3-n}{2})^2\Gamma(\frac{n+d-1}{2})}.
\end{equation}
Finally, combining~\eqref{eq:int_Gegen3} with~\eqref{eq:lamn} and~\eqref{eq:P_and_C} gives the results of $\lambda_n$.
\end{proof}

\begin{remark}
From Theorem~\ref{thm:specK} we know that the spectrum of $\cK$ is stage-like, and the stage becomes wider as the dimension $d$ gets larger. And when $d$ is fixed, the gap between stages is large. Specifically, by~\eqref{eq:Lambda} we have
\begin{equation}
\lambda_{n+2} = \frac{(n-1)^2}{(n+d-1)^2(n+d+1)(n+d)}\lambda_n \leq \frac{\lambda_n}{(n+d)^2}.
\end{equation}
\end{remark}

\subsection{The smallest eigenvalues}
To gain some insights about the small eigenvalues, we draw some inspiration from the well-known Marchenko-Pastur distribution, which characterizes the spectrum of random matrix $Y_n = \frac{1}{n}XX^T$, where $X\in\RR^{n\times m}$ with the entries being \textit{i.i.d.} random variables with mean $0$ and variance $1$. Let $\mu_n(\lambda) = \frac{1}{m}\sum_{j=1}^m \delta(\lambda-\lambda_j(Y_n))$ denote the spectrum of $Y_n$. Random matrix theory shows that $\mu_n$ converges to the the following Marchenko-Pastur (MP)~ \citep{marchenko1967distribution} distribution as $m,n\to \infty$, 
\[
\mu(\lambda) = \begin{cases}
(1-\frac{1}{\gamma}) \delta(\lambda) + v_{1/\gamma}(\lambda), & \text{if } \gamma>1 \\
v_{\gamma}(\lambda), & \text{if } \gamma \leq 1,
\end{cases}
\]
where 
\[
    v_{\gamma}(\lambda) = \frac{1}{2\pi} \frac{\sqrt{(\lambda_{+}-\lambda)(\lambda-\lambda_{-})}}{\lambda} \mathbf{1}_{\lambda\in [\lambda_{-}, \lambda_{+}]},
\]
and $\lambda_{\pm} = (1\pm\sqrt{\gamma})^2$. 

The smallest eigenvalue is given by $\lambda_{-}$.
When $\gamma=1$, we have $\lambda_{-}=0$ and $\mu(\lambda) \propto \frac{1}{\sqrt{\lambda}}$ at $\lambda\approx 0$. This implies that the smallest eigenvalues is zero and has a long tail. This is 
consistent with  the small part of the spectrum of the Gram matrix at $\gamma=1$ shown in Figure \ref{fig: spectrum-1}. 
Assuming that the MP distribution  can be used to characterize the tail, the smallest eigenvalue $\lambda_n(\gamma)$ of the Gram matrix at $\gamma\approx 1$ should
then  obey
 \begin{equation}\label{eqn: MP-prediction}
 \lambda_n(\gamma) \approx c_{n,d} 
     \begin{cases} 
     (1-\sqrt{\gamma})^2 &  \text{if } \gamma \leq 1 \\
     (1-\sqrt{1/\gamma})^2 & \text{if } \gamma > 1,
    \end{cases}
 \end{equation}
where $c_{n,d}$ is a constant that depends on $n,d$. The Figure \ref{fig: minium-eigvals} illustrates the difference between the true smallest eigenvalues and the prediction in \eqref{eqn: MP-prediction}. We see that the prediction is pretty accurate when $\gamma\approx 1$,.
% and  is also not that bad when $\gamma$ away from $1$.
\begin{figure}[!h]
\centering 
\includegraphics[width=0.5\textwidth]{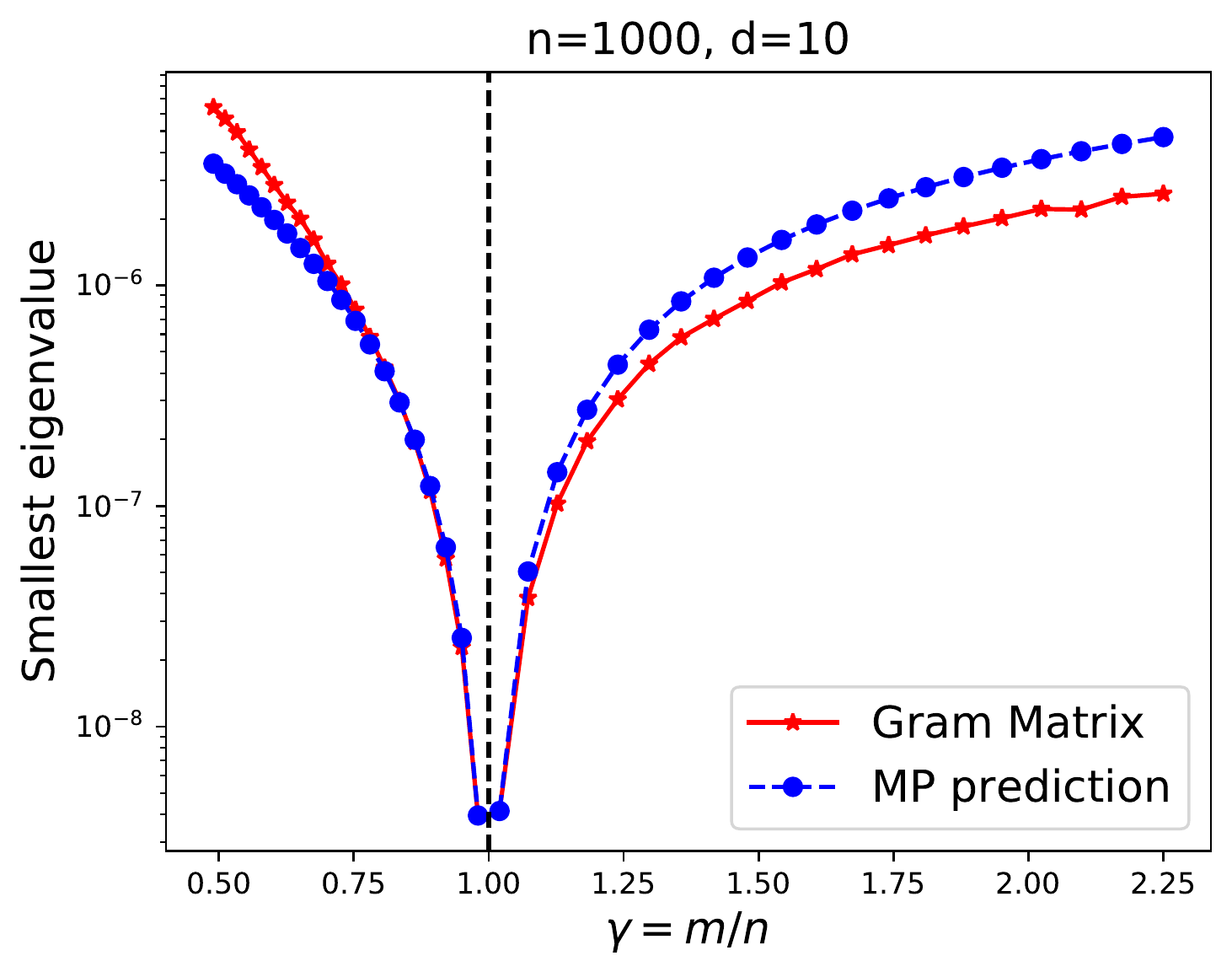} 
\caption{The smallest eigenvalues of the Gram matrices for different values of $\gamma$. In this experiment, $n=1000, d=10$, and for each $\gamma$, we report the mean values of $10$ independent experiments.}
\label{fig: minium-eigvals}
\end{figure}
\end{document}